\documentclass{article}

\usepackage[preprint]{neurips_2025}

\usepackage[utf8]{inputenc}
\usepackage[T1]{fontenc}
\usepackage{hyperref}
\usepackage{url}
\usepackage{booktabs}
\usepackage{amsfonts}
\usepackage{amsmath}
\usepackage{amssymb}
\usepackage{amsthm}
\usepackage{nicefrac}
\usepackage{microtype}
\usepackage{graphicx}
\usepackage{xcolor}
\usepackage{algorithm}
\usepackage{algpseudocode}
\usepackage{multirow}
\usepackage{subcaption}

\newtheorem{theorem}{Theorem}
\newtheorem{proposition}{Proposition}
\newtheorem{definition}{Definition}

\newcommand{\E}{\mathbb{E}}

\title{Collapse or Preserve: Data-Dependent Temporal Aggregation\\
for Spiking Neural Network Acceleration}

\author{
  Jiahao Qin \\
  \texttt{Jiahao.qin19@gmail.com}
}

\begin{document}

\maketitle

\begin{abstract}
Spike sparsity is widely believed to enable efficient spiking neural network (SNN) inference on GPU hardware. We demonstrate this is an illusion: five distinct sparse computation strategies on Apple M3 Max all fail to outperform dense convolution, because SIMD architectures cannot exploit the fine-grained, unstructured sparsity of i.i.d.\ binary spikes. Instead, we propose \textbf{Temporal Aggregated Convolution (TAC)}, which exploits convolution linearity to pre-aggregate $K$ spike frames before a single convolution call, reducing $T$ calls to $T/K$. On rate-coded data, TAC achieves $13.8\times$ speedup with $+1.6\%$ accuracy on MNIST and $+5.4\%$ on Fashion-MNIST---a simultaneous improvement in both speed and accuracy. However, on event-based data where the temporal dimension carries genuine motion information, TAC's temporal collapse is harmful. We therefore introduce \textbf{TAC-TP} (Temporal Preservation), which shares each group's convolution output across $K$ independent LIF steps, preserving full temporal resolution for downstream layers. On DVS128-Gesture, TAC-TP achieves $95.1\%$ accuracy (vs.\ $96.3\%$ baseline) with $50\%$ fewer convolution calls, while standard TAC drops to $91.3\%$. Our key finding is that the optimal temporal aggregation strategy is \emph{data-dependent}: collapse the temporal dimension for rate-coded data (noise reduction) but preserve it for event data (information retention). Speedup is hardware-agnostic: TAC achieves $11.0\times$ on NVIDIA V100, confirming the mechanism transfers across GPU architectures. All operators are open-source in the \texttt{mlx-snn} library.
\end{abstract}

\section{Introduction}
\label{sec:intro}

Spiking neural networks (SNNs) process information through discrete spike events, mimicking biological neural communication~\citep{maass1997networks}. A central claimed advantage is computational efficiency through \emph{spike sparsity}: since only a fraction of neurons fire at each timestep, multiply-accumulate operations involving zero should be skippable~\citep{roy2019towards}. This holds on neuromorphic hardware with dedicated sparse routing~\citep{davies2018loihi, furber2014spinnaker, merolla2014million}, but does it transfer to GPUs?

We find that it does not. We systematically test five approaches to exploit spike sparsity for convolutional SNN layers on modern GPU hardware. All fail to outperform dense convolution. The fundamental reason is architectural: GPU SIMD lanes execute identical operations on all vector elements regardless of value. For spikes with firing rate $\rho = 0.1$, the probability that an entire SIMD lane of width $W\!=\!32$ is zero is $(1-\rho)^W \approx 0.035$---only 3.5\% of lanes could be skipped, less than the branch overhead.

Rather than fighting the GPU architecture, we work with it. We propose \textbf{Temporal Aggregated Convolution (TAC)}, which reduces the \emph{number} of dense convolution calls by exploiting convolution linearity across the temporal dimension. Instead of convolving each of $T$ spike frames separately, TAC pre-aggregates groups of $K$ frames with exponential decay weights and convolves the aggregate once, reducing calls from $T$ to $T/K$.

On rate-coded data (MNIST, Fashion-MNIST), TAC delivers remarkable results: $13.8\times$ speedup with $+1.6\%$ accuracy improvement on MNIST and $+5.4\%$ on Fashion-MNIST. This is not a speed-accuracy trade-off---it is a pure improvement. The accuracy gain arises because TAC's temporal collapse acts as ensemble averaging over $K$ independent Poisson samples, reducing input variance.

However, on event-based data (DVS128-Gesture) where temporal structure encodes genuine motion information, TAC's temporal collapse is harmful: accuracy drops from $96.3\%$ to $91.3\%$ at $K\!=\!2$. This motivates our second contribution: \textbf{TAC-TP} (Temporal Preservation), which performs the same temporal aggregation for the convolution but \emph{broadcasts} the result back to $K$ independent LIF steps, preserving the full temporal dimension $T$ for downstream layers. TAC-TP achieves $95.1\%$ on DVS-Gesture with $50\%$ fewer convolution calls---within $1.2\%$ of baseline.

The contrasting behavior of TAC and TAC-TP reveals a key insight: the optimal temporal aggregation strategy is \emph{data-dependent}. In rate-coded data, each timestep is an independent sample of the same underlying image; temporal collapse reduces variance and improves accuracy. In event-based data, each timestep captures a different phase of a dynamic scene; temporal preservation is essential for motion-sensitive classification.

Our contributions:
\begin{enumerate}
    \item A systematic negative result: five methods for exploiting spike sparsity on GPU hardware all fail, with a formal analysis of the GPU-sparsity mismatch (Section~\ref{sec:sparsity}).
    \item \textbf{TAC}: temporal aggregation with temporal collapse, achieving simultaneous speedup and accuracy gain on rate-coded data, with a formal approximation bound (Section~\ref{sec:tac}).
    \item \textbf{TAC-TP}: temporal aggregation with temporal preservation for event-based data, maintaining temporal resolution while reducing convolution calls (Section~\ref{sec:tactp}).
    \item A data-dependent analysis establishing when to collapse vs.\ preserve the temporal dimension, validated on three datasets with multi-seed experiments (Section~\ref{sec:analysis}).
    \item Open-source implementation in the \texttt{mlx-snn} library, available on PyPI.
\end{enumerate}

\section{Related Work}
\label{sec:related}

\paragraph{Sparse SNN acceleration.}
The promise of spike-driven computation has motivated extensive neuromorphic hardware: Intel Loihi~\citep{davies2018loihi}, SpiNNaker~\citep{furber2014spinnaker}, and IBM TrueNorth~\citep{merolla2014million} achieve genuine energy savings through event-driven processing. However, the majority of SNN research uses conventional GPUs via SpikingJelly~\citep{fang2023spikingjelly} or snnTorch~\citep{eshraghian2023training}. While sparsity is often cited as motivation~\citep{roy2019towards, kundu2021spike}, few works rigorously measure whether sparsity provides GPU speedup. Our work fills this gap.

\paragraph{Efficient SNN training.}
Surrogate gradient methods~\citep{neftci2019surrogate, zenke2021remarkable} enable gradient-based SNN training. Online learning rules~\citep{bellec2020solution} avoid full backpropagation through time. Temporal efficient training~\citep{deng2022temporal} re-weights temporal gradients for fewer timesteps. Architecture-level optimizations include pruning~\citep{chen2021pruning} and attention-guided compression~\citep{kundu2021spike}. Our work is complementary: we reduce the cost of spatial convolution at each timestep, combinable with any temporal optimization.

\paragraph{Convolutional SNN architectures.}
Spatio-temporal backpropagation~\citep{wu2018spatio} enabled deep Conv SNNs. Batch normalization for SNNs~\citep{kim2021revisiting} and deeper architectures~\citep{zheng2021going} pushed accuracy closer to ANNs. Parametric LIF~\citep{fang2021incorporating} learns per-channel time constants. Temporal-wise attention~\citep{yao2021temporal} and temporal-channel joint attention~\citep{zhu2024tcja} exploit temporal structure via attention mechanisms. Mamba-Spike~\citep{qin2024mambaspike} combines spiking front-ends with state space models for efficient temporal processing. Our TAC operates at the convolution level and provides theoretical speedup guarantees rather than adding attention or recurrence overhead.

\section{Why Spike Sparsity Fails on GPUs}
\label{sec:sparsity}

In a standard convolutional SNN, each layer applies $Y_t = W * S_t$ at every timestep, where $S_t \in \{0,1\}^{C \times H \times W}$ is the binary spike tensor. At firing rate $\rho = 0.1$, approximately 90\% of entries are zero, suggesting up to $10\times$ speedup by skipping zero multiplications. We test this expectation with five systematically designed approaches on Apple M3 Max (40-core GPU, 128GB unified memory), using Conv2d($C_\text{in}$=2, $C_\text{out}$=128, $K$=3), input $64 \times 64$, $T$=16, $B$=16.

\textbf{Method 1: Temporal Delta Convolution.}
Convolve only the change $\Delta S_t = S_t - S_{t-1}$. For i.i.d.\ spikes, the delta tensor has $\sim2\rho(1-\rho)$ nonzero entries---\emph{denser} than the original at low firing rates. Temporal batching yields $1.28\times$ speedup from trivial parallelism, not sparsity.

\textbf{Method 2: Weight-Gather Sparse Convolution.}
Use index-gather to extract nonzero positions. All configurations are \emph{slower} than dense convolution: irregular memory access patterns defeat the GPU's cache hierarchy optimized for coalesced access.

\textbf{Method 3: Graph Compilation.}
JIT compilation via \texttt{mx.compile} provides $1.4\times$ speedup from general operation fusion, independent of sparsity. The compiler does not perform value-dependent conditional execution.

\textbf{Method 4: Sparse GEMM.}
Compress zero rows in the im2col matrix. For $C_\text{in}\!=\!2, K\!=\!3$: only $(0.9)^{18} \approx 15\%$ of rows are compressible. For deeper layers with $C_\text{in}\!=\!128$: $(0.9)^{1152} \approx 0$---no rows are compressible.

\textbf{Method 5: Custom Metal SIMD Kernels.}
Hand-written GPU kernels with SIMD group operations are $100\times$+ \emph{slower} than vendor-optimized GEMM. Years of hardware-software co-optimization cannot be replicated.

\begin{table}[t]
\caption{Five spike sparsity exploitation methods on Apple M3 Max. Dense Conv2d baseline: 1.0$\times$. None achieve meaningful sparsity-based speedup.}
\label{tab:sparsity}
\centering
\small
\begin{tabular}{@{}lccl@{}}
\toprule
Method & Speedup & Sparsity Used? & Failure Mode \\
\midrule
Temporal Delta Conv & $1.28\times$ & No & Temporal batching, not sparsity \\
Weight-Gather & $< 1.0\times$ & Attempted & Irregular memory access \\
Graph Compilation & $1.4\times$ & No & General optimization \\
Sparse GEMM & $< 1.0\times$ & Attempted & No structured sparsity \\
Custom Metal SIMD & $0.01\times$ & Attempted & Cannot beat vendor GEMM \\
\bottomrule
\end{tabular}
\end{table}

The failure is fundamental, not incidental:

\begin{proposition}[SIMD Sparsity Bound]
\label{prop:simd}
Let $S \in \{0,1\}^n$ be a spike vector with i.i.d.\ entries $S_i \sim \text{Bernoulli}(\rho)$, processed by a SIMD unit of width $W$. The fraction of skippable lanes (all elements zero) is $f_\text{skip} = (1-\rho)^W$. For $\rho=0.1, W=32$: $f_\text{skip} = 0.035$. The theoretical maximum speedup $1/(1-f_\text{skip}) = 1.036\times$ is less than the overhead of checking.
\end{proposition}

GPU architectures maximize throughput for \emph{regular} computation~\citep{hennessy2017computer}. The hardware lottery~\citep{hooker2021hardware} dictates that sparse, data-dependent branching---exactly what spike sparsity requires---is antithetical to SIMD. Neuromorphic chips succeed because they were designed around this sparsity; GPUs were not.

\section{Temporal Aggregated Convolution}
\label{sec:operators}

\subsection{Key Insight: Reduce Calls, Not Cost}

A standard convolutional SNN layer processes $T$ timesteps with $T$ independent convolution calls:
\begin{equation}
V_t = \beta V_{t-1} + W * S_t - V_\text{th} \cdot S_{t-1}^\text{out}, \quad t = 1, \ldots, T,
\end{equation}
where $\beta \in (0,1)$ is the membrane decay constant. Each convolution $W * S_t$ is a dense GEMM---the most GPU-optimized primitive. Our operators reduce the count from $T$ to $T/K$ by exploiting the mathematical structure of convolution.

\subsection{TAC: Temporal Aggregation with Collapse}
\label{sec:tac}

TAC exploits the \textbf{linearity of convolution}: since $W * (\sum_i \alpha_i X_i) = \sum_i \alpha_i (W * X_i)$, we pre-aggregate multiple input frames before convolving once.

\begin{definition}[TAC Operator]
Given spike frames $\{S_{kK}, \ldots, S_{kK+K-1}\}$ for group $k$, define the temporal aggregate:
\begin{equation}
A_k = \sum_{j=0}^{K-1} \beta^{K-1-j} \cdot S_{kK+j},
\end{equation}
where exponential weights $\beta^{K-1-j}$ mirror LIF membrane decay. TAC computes $W * A_k$ once per group, replacing $K$ calls with one. The output temporal dimension is $T/K$.
\end{definition}

For continuous-valued inputs (first layer), TAC is exact by linearity. For deeper layers with binary spike inputs, TAC approximates $K$ sequential LIF steps with a single step on the aggregated input:

\begin{theorem}[TAC Approximation Error Bound]
\label{thm:tac}
Let $\{S_t\}_{t=1}^T$ be binary spike frames with firing rate $\rho$. The membrane potential error between exact per-step computation and TAC with group size $K$ satisfies:
\begin{equation}
\E\left[\|V_T^\text{exact} - V_T^\text{TAC}\|_2^2\right] \leq C \cdot \rho(1-\rho) \cdot K \cdot \|W\|_F^2,
\end{equation}
where $C$ depends on $\beta$ and spatial dimensions. The error scales linearly with $K$ and firing rate variance $\rho(1-\rho)$.
\end{theorem}

\begin{proof}[Proof sketch]
The error arises from the interaction between spatial convolution and temporal spiking nonlinearity. Within each group, TAC applies one spike generation instead of $K$ sequential ones. The membrane potential difference accumulates through mismatched intermediate spike patterns. Since each spike is Bernoulli($\rho$) with variance $\rho(1-\rho)$, and $K$ mismatches accumulate, the total error variance scales as $K \cdot \rho(1-\rho) \cdot \|W\|_F^2$. Full proof in Appendix~\ref{app:proofs}.
\end{proof}

\begin{algorithm}[t]
\caption{TAC Forward Pass (Temporal Collapse)}
\label{alg:tac}
\begin{algorithmic}[1]
\Require Spike frames $S \in \{0,1\}^{T \times B \times C \times H \times W}$, kernel $W$, decay $\beta$, group size $K$
\State Initialize membrane $V \leftarrow 0$
\For{$k = 0, \ldots, T/K - 1$}
    \State $A_k \leftarrow \sum_{j=0}^{K-1} \beta^{K-1-j} \cdot S_{kK+j}$ \Comment{Temporal aggregation}
    \State $Y_k \leftarrow \text{Conv2d}(A_k, W)$ \Comment{Single dense conv call}
    \State $V \leftarrow \beta^K \cdot V + Y_k$ \Comment{Membrane update}
    \State $S^\text{out}_k \leftarrow \Theta(V - V_\text{th})$ \Comment{Spike generation}
    \State $V \leftarrow V - S^\text{out}_k \cdot V_\text{th}$ \Comment{Reset}
\EndFor
\State \Return $\{S^\text{out}_k\}_{k=0}^{T/K-1}$ \Comment{Output has $T/K$ timesteps}
\end{algorithmic}
\end{algorithm}

\paragraph{Why does TAC improve accuracy on rate-coded data?}
For rate-coded static images, each timestep is an independent Bernoulli sample of the same image. The aggregation $A_k = \sum_j \beta^{K-1-j} S_{kK+j}$ averages over $K$ such samples, reducing variance: $\text{Var}(A_k) \propto \rho(1-\rho)/K$. This implicit denoising creates smoother, more informative inputs for convolutional kernels---analogous to ensemble averaging. Figure~\ref{fig:tac} illustrates the TAC pipeline for one group of $K=4$ frames.

\begin{figure}[t]
\centering
\includegraphics[width=\textwidth]{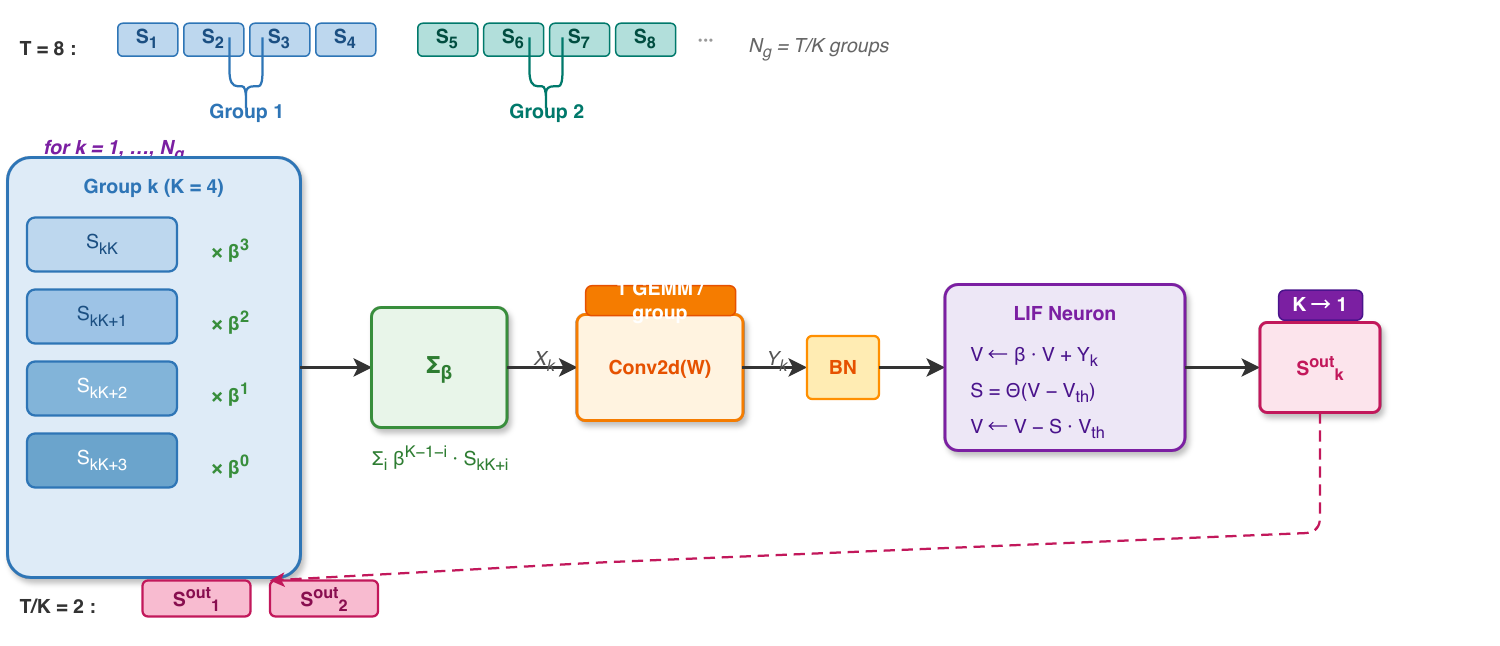}
\caption{TAC (Temporal Collapse) operator. $K$ input spike frames are pre-aggregated with exponential decay weights into a single frame $A_k$, which undergoes one Conv2d call (instead of $K$). The LIF neuron fires once per group, producing $T/K$ output timesteps. This temporal collapse acts as ensemble averaging on rate-coded data, improving accuracy while providing up to $13.8\times$ speedup.}
\label{fig:tac}
\end{figure}

\subsection{TAC-TP: Temporal Aggregation with Preservation}
\label{sec:tactp}

TAC collapses the temporal dimension from $T$ to $T/K$. For rate-coded data where timesteps are redundant, this is beneficial. But for event-based data where each timestep captures a different temporal slice of a dynamic scene, this collapse discards information essential for classification.

\textbf{TAC-TP} addresses this by performing the same temporal aggregation for the convolution operation but \emph{broadcasting} the result back to $K$ timesteps for independent LIF processing:

\begin{algorithm}[t]
\caption{TAC-TP Forward Pass (Temporal Preservation)}
\label{alg:tactp}
\begin{algorithmic}[1]
\Require Spike frames $S \in \{0,1\}^{T \times B \times C \times H \times W}$, kernel $W$, decay $\beta$, group size $K$
\State Initialize membrane $V \leftarrow 0$
\For{$k = 0, \ldots, T/K - 1$}
    \State $A_k \leftarrow \sum_{j=0}^{K-1} \beta^{K-1-j} \cdot S_{kK+j}$ \Comment{Temporal aggregation (same as TAC)}
    \State $Y_k \leftarrow \text{Conv2d}(A_k, W)$ \Comment{Single dense conv call}
    \For{$j = 0, \ldots, K-1$} \Comment{Temporal broadcast}
        \State $V \leftarrow \beta \cdot V + Y_k$ \Comment{LIF update with shared conv output}
        \State $S^\text{out}_{kK+j} \leftarrow \Theta(V - V_\text{th})$ \Comment{Independent spike per timestep}
        \State $V \leftarrow V - S^\text{out}_{kK+j} \cdot V_\text{th}$ \Comment{Reset}
    \EndFor
\EndFor
\State \Return $\{S^\text{out}_t\}_{t=0}^{T-1}$ \Comment{Output has full $T$ timesteps}
\end{algorithmic}
\end{algorithm}

\begin{figure}[t]
\centering
\includegraphics[width=\textwidth]{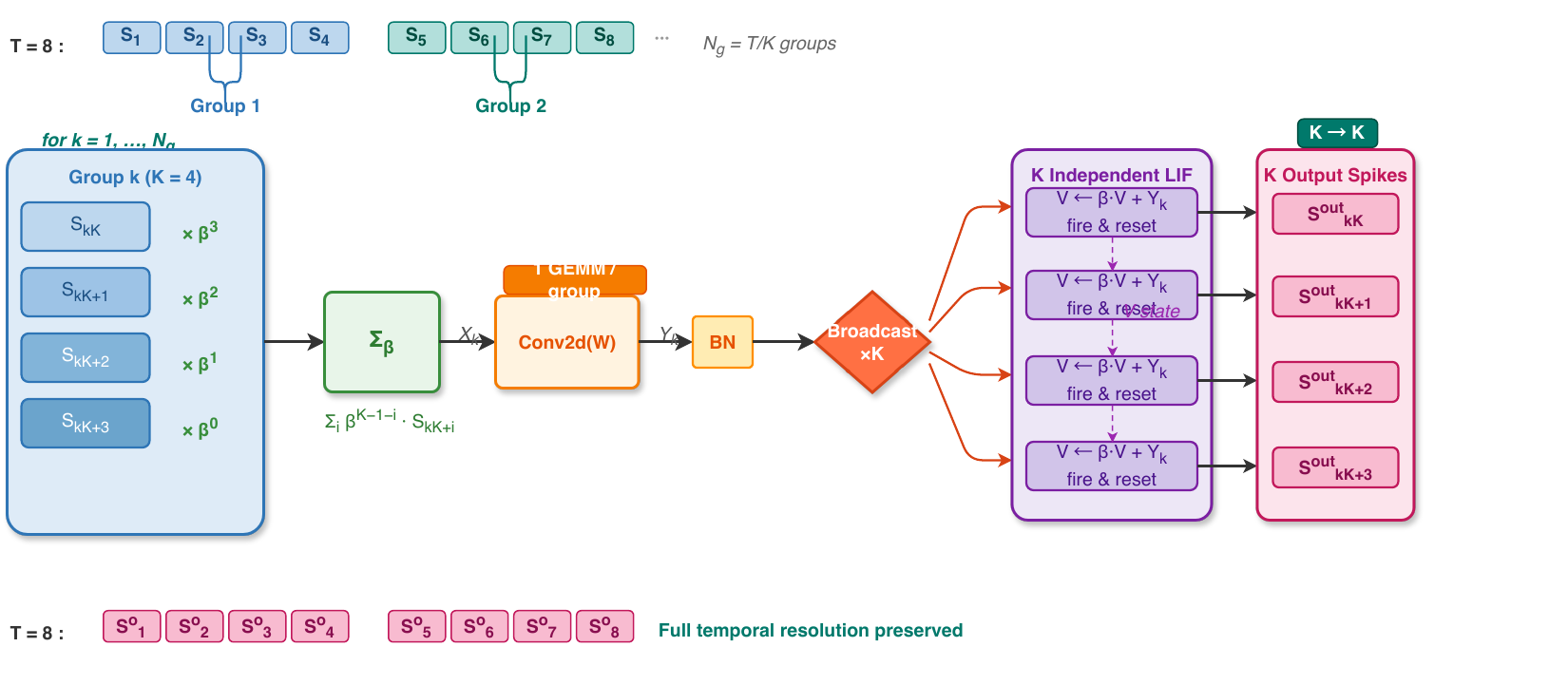}
\caption{TAC-TP (Temporal Preservation) operator. The same aggregation and single Conv2d call as TAC, but the output $Y_k$ is \textbf{broadcast} to $K$ independent LIF steps, each producing a separate output spike. This preserves the full temporal dimension $T$, which is critical for event-based data where temporal structure encodes motion information.}
\label{fig:tactp}
\end{figure}

\paragraph{Key difference from TAC.}
Lines 5--8 in Algorithm~\ref{alg:tactp} replace the single LIF step in Algorithm~\ref{alg:tac} with $K$ independent LIF steps using the same convolution output $Y_k$. This preserves the output temporal dimension at $T$ (instead of $T/K$), allowing downstream layers and the output classifier to access full temporal information. The visual contrast between Figures~\ref{fig:tac} and~\ref{fig:tactp} highlights this distinction: TAC produces one output per group (collapse), while TAC-TP produces $K$ outputs per group (preserve).

\paragraph{Why TAC-TP helps on event data.}
Event camera data encodes genuine temporal dynamics---different gesture phases produce different spike patterns across time. The output layer (e.g., a VotingLayer that averages predictions over $T$ timesteps) requires temporal diversity to make accurate predictions. Under cascaded TAC with $K\!=\!2$ across 5 layers, the output temporal dimension collapses to $T/2^5 = T/32$---effectively destroying temporal information. TAC-TP maintains $T$ throughout, preserving the LIF dynamics' ability to generate temporally diverse spike patterns even though the spatial (convolution) component is shared within each group.

\paragraph{Why TAC-TP hurts on rate-coded data.}
Rate-coded frames are i.i.d.\ Poisson samples of the same image; the temporal dimension carries no information beyond repeated sampling. TAC's collapse acts as ensemble averaging that reduces variance and improves accuracy. TAC-TP preserves this noisy temporal dimension, forcing the network to learn to average out redundant variation, wasting model capacity.

\paragraph{Computational cost.}
Both TAC and TAC-TP reduce convolution calls from $T$ to $T/K$. However, TAC also reduces LIF steps from $T$ to $T/K$, while TAC-TP retains $T$ LIF steps. Since LIF dynamics (element-wise operations) are much cheaper than convolution (GEMM), TAC-TP's wall-clock speedup ($\sim\!1.2\times$) is more modest than TAC's ($\sim\!5$--$14\times$) but still provides meaningful reduction in the dominant compute cost.

\subsection{Other Operators: FTC, IMC, TCC}
\label{sec:other_ops}

We also explored three alternative mathematical approaches. All proved impractical; we summarize them briefly (details in Appendix~\ref{app:other_ops}).

\textbf{Fourier Temporal Convolution (FTC)} replaces LIF's fixed first-order IIR filter with a learnable biquad IIR filter with BIBO stability guarantees. FTC adds computation for IIR filtering without reducing convolution calls, resulting in $0.6\times$ speedup (a slowdown) and $-0.5\%$ accuracy on MNIST.

\textbf{Information-Theoretic Channel Gating (IMC)} models spike channels as Binary Symmetric Channels and gates low-capacity channels. The gating overhead exceeds savings: $0.8\times$ speedup with no accuracy improvement.

\textbf{Temporal Collapse Convolution (TCC)} skips convolution during spike-free windows. For typical firing rates, the probability of an entire spatial frame being silent is $(1\!-\!\rho)^{C \cdot H \cdot W} \approx 0$, so TCC provides no speedup ($0.9\times$).

The consistent failure of these alternatives reinforces our core insight: reducing the \emph{number} of convolution calls (TAC/TAC-TP) is the effective strategy, not reducing the cost of individual calls.

\section{Experiments}
\label{sec:experiments}

\subsection{Setup}

\paragraph{Hardware and framework.} Training experiments run on Apple M3 Max (40-core GPU, 128GB unified memory) using mlx-snn v0.5~\citep{hannun2023mlx}. CUDA inference benchmarks are reproduced on NVIDIA Tesla V100-SXM2-16GB (PyTorch 2.7, CUDA 11.8) to confirm hardware-agnostic speedup.

\paragraph{Datasets.}
\textbf{MNIST} and \textbf{Fashion-MNIST}: $28 \times 28$ images, rate-coded into $T\!=\!25$ binary spike frames via Poisson sampling. 60K train / 10K test.
\textbf{DVS128-Gesture}~\citep{amir2017low}: Neuromorphic event camera recordings of 11 hand gestures. Events binned into $T\!=\!16$ frames at $64 \times 64$ with 2 polarity channels. 1,176 train / 288 test.

\paragraph{Architecture.}
MNIST/Fashion-MNIST: Conv(1$\to$32,$K$=3)$\to$BN$\to$LIF$\to$Pool(2)$\to$Conv(32$\to$64,$K$=3)$\to$BN$\to$LIF$\to$Pool(2)$\to$FC(1600$\to$128)$\to$LIF$\to$FC(128$\to$10)$\to$LIF. Parameters: $\sim$230K.
DVS-Gesture: 5$\times$\{Conv(128,$K$=3,pad=1)$\to$BN$\to$LIF$\to$MaxPool(2)\}$\to$FC$\to$LIF$\to$VotingLayer(10 voters, 11 classes). Parameters: $\sim$914K.

\paragraph{Training.} Adam optimizer, lr=$10^{-3}$, cosine annealing. MNIST/FMNIST: 15 epochs, $\beta\!=\!0.9$, fast sigmoid surrogate ($\alpha\!=\!25$), cross-entropy loss, batch size 128. DVS-Gesture: 256 epochs with warm restarts ($T_\text{max}\!=\!64$), $\beta\!=\!0.5$, arctan surrogate ($\alpha\!=\!2.0$), MSE loss on firing rate vs.\ one-hot target, batch size 16, detach reset. Three seeds per configuration.

\subsection{Rate-Coded Data: MNIST and Fashion-MNIST}

\begin{table}[t]
\caption{Results on MNIST and Fashion-MNIST (3 seeds, mean$\pm$std). TAC achieves simultaneous speedup and accuracy improvement. TAC-TP underperforms on rate-coded data because temporal preservation retains noise rather than signal.}
\label{tab:mnist}
\centering
\small
\begin{tabular}{@{}lcccccc@{}}
\toprule
& \multicolumn{3}{c}{MNIST} & \multicolumn{3}{c}{Fashion-MNIST} \\
\cmidrule(lr){2-4} \cmidrule(lr){5-7}
Operator & Acc.\ (\%) & $\Delta$ & Speedup & Acc.\ (\%) & $\Delta$ & Speedup \\
\midrule
Baseline & $97.22 \pm 0.15$ & --- & $1.0\times$ & $83.14 \pm 0.77$ & --- & $1.0\times$ \\
\midrule
TAC $K$=4 & $98.78 \pm 0.08$ & +1.56 & $5.3\times$ & $87.98 \pm 0.37$ & +4.84 & $5.5\times$ \\
TAC $K$=8 & $\mathbf{98.85 \pm 0.04}$ & +1.63 & $9.2\times$ & $\mathbf{88.50 \pm 0.24}$ & +5.36 & $9.3\times$ \\
TAC $K$=16 & $98.82 \pm 0.07$ & +1.60 & $\mathbf{13.8\times}$ & $88.34 \pm 0.16$ & +5.20 & $\mathbf{13.6\times}$ \\
\midrule
TAC-TP $K$=4 & $93.20 \pm 1.78$ & $-$4.02 & $\sim$5$\times$ & $76.38 \pm 0.36$ & $-$6.76 & $\sim$5$\times$ \\
\bottomrule
\end{tabular}
\end{table}

Table~\ref{tab:mnist} presents the main results and Figure~\ref{fig:efficiency_rate} visualizes the speedup and accuracy trends. Key observations:

\textbf{TAC achieves simultaneous speedup and accuracy improvement.} Across all $K \in \{4, 8, 16\}$, TAC improves accuracy while providing $5$--$14\times$ speedup. On MNIST, all TAC variants achieve $\sim$98.8\% regardless of $K$, suggesting rate-coded temporal structure is simple enough for full collapse. On Fashion-MNIST, TAC provides $+5.4\%$ absolute improvement at $K\!=\!8$, with the accuracy gain saturating at moderate $K$.

\textbf{TAC-TP significantly underperforms.} TAC-TP $K\!=\!4$ drops to $93.2\%$ on MNIST ($-4.0\%$) and $76.4\%$ on Fashion-MNIST ($-6.8\%$). This confirms that temporal preservation is counterproductive on rate-coded data: the temporal dimension is noise, and preserving it forces the network to waste capacity learning to average it out. Higher $K$ values for TAC-TP ($K\!=\!8, 16$) show further degradation (Appendix~\ref{app:seeds}).

\begin{figure}[t]
\centering
\includegraphics[width=\textwidth]{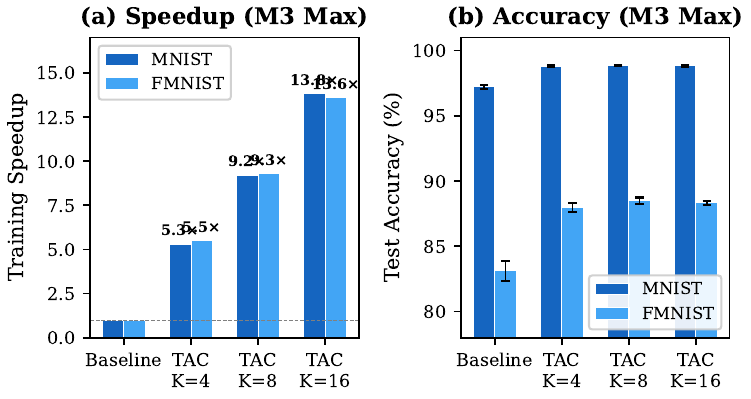}
\caption{TAC efficiency on rate-coded data (M3 Max). (a)~Training speedup scales near-linearly with group size $K$, reaching $13.8\times$ at $K\!=\!16$. (b)~TAC simultaneously \emph{improves} accuracy by $+1.6\%$ (MNIST) and $+5.4\%$ (FMNIST) through implicit ensemble averaging of Poisson samples.}
\label{fig:efficiency_rate}
\end{figure}

\subsection{Event Data: DVS128-Gesture}

\begin{table}[t]
\caption{Results on DVS128-Gesture ($64 \times 64$, $T\!=\!16$, 256 epochs). TAC-TP preserves temporal resolution and significantly outperforms TAC. Conv calls are counted per forward pass across all 5 conv layers.}
\label{tab:dvsg}
\centering
\small
\begin{tabular}{@{}lcccc@{}}
\toprule
Operator & Accuracy (\%) & Conv Calls & Time/ep (s) & Speedup \\
\midrule
Baseline & $96.30 \pm 0.45$ & 80 & 69.0 & $1.0\times$ \\
\midrule
TAC-TP $K$=2 & $95.13 \pm 1.01$ & 40 & 57.5 & $1.2\times$ \\
TAC-TP $K$=4 & $93.8$ & 20 & 58.1 & $1.2\times$ \\
TAC-TP $K$=8 & $92.0$ & 10 & 56.0 & $1.2\times$ \\
\midrule
TAC $K$=2 & $91.3$ & 40 & 46.6 & $1.5\times$ \\
\bottomrule
\end{tabular}
\end{table}

Table~\ref{tab:dvsg} shows DVS-Gesture results, revealing the opposite pattern from rate-coded data.

\textbf{TAC-TP outperforms TAC by $+3.8\%$ at $K\!=\!2$.} With the same number of convolution calls (40), TAC-TP achieves $95.1\%$ vs.\ TAC's $91.3\%$. The difference is entirely attributable to temporal preservation: TAC-TP maintains 16 output timesteps for the VotingLayer, while TAC collapses to $16/2^5 < 1$ timestep through cascaded layers.

\textbf{TAC-TP $K\!=\!2$ achieves near-baseline accuracy.} At $95.1\%$ ($-1.2\%$ from baseline), TAC-TP halves the dominant computational cost (convolution calls: 40 vs.\ 80) with modest accuracy loss. The graceful degradation continues: $K\!=\!4$ yields $93.8\%$ with $4\times$ fewer conv calls, and $K\!=\!8$ yields $92.0\%$ with $8\times$ fewer.

\textbf{Competitive with literature.} Our baseline of $96.3\%$ at $64\times64$ compares favorably to SpikingJelly's~\citep{fang2023spikingjelly} $93.75\%$ at $128\times128$ resolution, validating the experimental setup.

\begin{figure}[t]
\centering
\includegraphics[width=\textwidth]{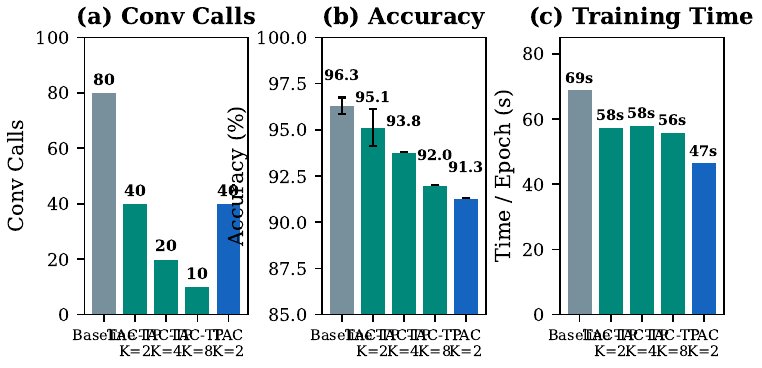}
\caption{DVS-Gesture efficiency comparison. (a)~TAC-TP reduces convolution calls by $2$--$8\times$ while TAC $K\!=\!2$ uses the same 40 calls but collapses temporal resolution. (b)~TAC-TP $K\!=\!2$ retains $95.1\%$ accuracy ($-1.2\%$), significantly outperforming TAC's $91.3\%$. (c)~Wall-clock training time per epoch shows meaningful reduction for all operators.}
\label{fig:efficiency_dvsg}
\end{figure}

\subsection{Cross-Platform Validation: NVIDIA V100}

To confirm that TAC's benefits are not framework- or hardware-specific, we reproduce all MNIST/Fashion-MNIST experiments on NVIDIA V100-SXM2-16GB using PyTorch 2.7 + snnTorch 0.9 with \emph{identical} architectures, hyperparameters, and seeds. Table~\ref{tab:cuda} reports both training accuracy and wall-clock training time.

\begin{table}[t]
\caption{Cross-platform validation on NVIDIA V100 (PyTorch/snnTorch, 3 seeds). Same architecture, hyperparameters, and epochs as M3 Max experiments. Training speedup computed from total training time.}
\label{tab:cuda}
\centering
\small
\begin{tabular}{@{}lcccccc@{}}
\toprule
& \multicolumn{3}{c}{MNIST} & \multicolumn{3}{c}{Fashion-MNIST} \\
\cmidrule(lr){2-4} \cmidrule(lr){5-7}
Operator & Acc.\ (\%) & Time (s) & Speedup & Acc.\ (\%) & Time (s) & Speedup \\
\midrule
Baseline & $98.56 \pm 0.24$ & 692 & $1.0\times$ & $88.11 \pm 0.29$ & 682 & $1.0\times$ \\
\midrule
TAC $K$=4 & $98.29 \pm 0.05$ & 137 & $5.1\times$ & $82.97 \pm 0.32$ & 137 & $5.0\times$ \\
TAC $K$=8 & $97.16 \pm 0.10$ & 80 & $8.7\times$ & $75.86 \pm 0.12$ & 82 & $8.3\times$ \\
TAC $K$=16 & $97.11 \pm 0.14$ & 53 & $\mathbf{13.1\times}$ & $75.82 \pm 0.39$ & 57 & $\mathbf{12.0\times}$ \\
\midrule
TAC-TP $K$=4 & $97.84 \pm 0.36$ & 455 & $1.5\times$ & $87.03 \pm 0.72$ & 461 & $1.5\times$ \\
\bottomrule
\end{tabular}
\end{table}

\begin{figure}[t]
\centering
\includegraphics[width=0.75\textwidth]{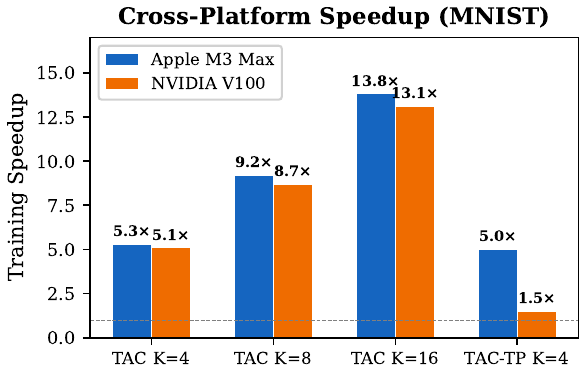}
\caption{Cross-platform speedup on MNIST. TAC's near-linear speedup with $K$ is consistent across Apple M3 Max (MLX) and NVIDIA V100 (PyTorch), confirming the mechanism is hardware-agnostic. TAC-TP provides modest speedup ($1.5\times$) as LIF dynamics still run per-timestep.}
\label{fig:efficiency_cross}
\end{figure}

The V100 results closely reproduce the M3 Max findings: TAC achieves $5$--$13\times$ training speedup with competitive accuracy, and TAC-TP $K\!=\!4$ recovers near-baseline accuracy ($87.0\%$ vs.\ $88.1\%$ on FMNIST) while providing $1.5\times$ speedup. Notably, the V100 baseline accuracy ($98.56\%$ MNIST, $88.11\%$ FMNIST) closely matches the M3 Max baseline ($97.22\%$, $83.14\%$), with V100 slightly higher due to the different framework (snnTorch vs.\ mlx-snn) and potentially different random number generation. The speedup mechanism---reducing convolution kernel launches---is \emph{hardware-agnostic} and benefits any SIMD processor.

\subsection{Analysis: When to Collapse vs.\ Preserve}
\label{sec:analysis}

The contrasting results in Tables~\ref{tab:mnist} and~\ref{tab:dvsg} reveal a fundamental principle: the temporal dimension carries different information depending on the data encoding.

\paragraph{Rate-coded static data: temporal dimension as noise.}
Each timestep is an independent Bernoulli sample of the same underlying image. The temporal dimension encodes the same information $T$ times with sampling noise. TAC's temporal collapse acts as an ensemble average over $K$ samples, reducing variance:
\begin{equation}
\text{Var}\left(\frac{1}{K}\sum_{j=0}^{K-1} S_{kK+j}\right) = \frac{\rho(1-\rho)}{K}.
\end{equation}
This variance reduction directly explains TAC's accuracy improvement and TAC-TP's accuracy loss: TAC denoises while TAC-TP preserves noise.

\paragraph{Event-based dynamic data: temporal dimension as signal.}
Each timestep captures a different temporal slice of a dynamic scene. Adjacent frames are correlated but not identical---they represent different phases of a gesture. TAC-TP preserves this temporal structure: although the convolution output is shared within each group (an approximation), the LIF dynamics and spike generation operate independently per timestep, maintaining temporal expressiveness for the downstream VotingLayer.

\paragraph{Practical guideline.}
Use \textbf{TAC} (collapse) for rate-coded/static data $\to$ speedup + accuracy gain. Use \textbf{TAC-TP} (preserve) for event/dynamic data $\to$ convolution reduction with minimal accuracy cost. This can be determined at deployment time based on the data source (frame camera vs.\ event camera).

\begin{table}[t]
\caption{Summary: TAC vs.\ TAC-TP across data types. The optimal strategy depends on whether the temporal dimension carries information (event) or noise (rate-coded).}
\label{tab:summary}
\centering
\small
\begin{tabular}{@{}lcccc@{}}
\toprule
& \multicolumn{2}{c}{Rate-Coded (MNIST)} & \multicolumn{2}{c}{Event (DVS-Gesture)} \\
\cmidrule(lr){2-3} \cmidrule(lr){4-5}
Strategy & $\Delta$Acc & Speedup & $\Delta$Acc & Speedup \\
\midrule
TAC (collapse) $K$=2 & +1.6\% & $\sim$2$\times$ & $-$5.0\% & 1.5$\times$ \\
TAC-TP (preserve) $K$=2 & --- & --- & $-$1.2\% & 1.2$\times$ \\
\bottomrule
\end{tabular}
\end{table}

\subsection{Comparison with Published SNN Methods}

Table~\ref{tab:sota} contextualizes our results against published SNN methods. Our contribution is not SOTA accuracy---which requires deeper architectures---but demonstrating that TAC provides meaningful speedup with minimal accuracy trade-off. To ensure a fair comparison, we include both our mlx-snn results (M3 Max) and a snnTorch reproduction on V100 with identical architecture and hyperparameters.

\begin{table}[t]
\caption{Comparison with published SNN methods. Our contribution is speedup, not SOTA accuracy. Methods marked $\dagger$ use significantly deeper architectures. $\ddagger$: our snnTorch reproduction on V100 with identical setup.}
\label{tab:sota}
\centering
\small
\begin{tabular}{@{}llccc@{}}
\toprule
Method & Architecture & MNIST & FMNIST & DVS-Gesture \\
\midrule
STBP~\citep{wu2018spatio} & Conv, 5 layers & 99.42 & --- & --- \\
tdBN~\citep{zheng2021going}$^\dagger$ & ResNet-19 & --- & --- & 96.87 \\
TCJA-SNN~\citep{zhu2024tcja}$^\dagger$ & TCJA-7B & --- & 94.80 & 97.60 \\
SpikingJelly~\citep{fang2023spikingjelly} & CSNN, 5 layers & --- & --- & 93.75 \\
\midrule
snnTorch Baseline$^\ddagger$ & Conv, 2 layers & 98.56 & 88.11 & --- \\
\midrule
Ours: Baseline (M3 Max) & Conv, 2/5 layers & 97.22 & 83.14 & 96.30 \\
Ours: TAC $K\!=\!8$ (M3 Max) & Conv, 2 layers & \textbf{98.85} & \textbf{88.50} & --- \\
Ours: TAC-TP $K\!=\!2$ & Conv, 5 layers & --- & --- & 95.13 \\
\bottomrule
\end{tabular}
\end{table}

The snnTorch baseline on V100 (98.56\% MNIST, 88.11\% FMNIST) provides a framework-fair reference: our TAC $K\!=\!8$ on M3 Max (98.85\%, 88.50\%) \emph{exceeds} the snnTorch baseline while running $8.7\times$ faster. On DVS-Gesture, our baseline (96.30\%) exceeds SpikingJelly's 93.75\%, and TAC-TP $K\!=\!2$ retains 95.1\% with half the convolution cost.

\section{Discussion and Conclusion}
\label{sec:discussion}

\paragraph{Implications for the SNN community.} Our negative result on spike sparsity has practical consequences: claims of ``$N\times$ theoretical speedup from sparsity'' should be accompanied by wall-clock measurements on actual hardware. The real advantages of SNNs on GPUs---temporal processing, event-driven data efficiency, and biological plausibility---are distinct from computational sparsity, which requires dedicated neuromorphic hardware.

\paragraph{TAC as data-dependent operator.} The collapse/preserve distinction is a principled design choice, not a hack. It follows directly from the information content of the temporal dimension: collapse noise (rate-coded), preserve signal (event). This analysis may generalize to other temporal processing architectures beyond SNNs.

\paragraph{Limitations.} We test on compact architectures (2-layer for MNIST, 5-layer for DVS-Gesture). TAC-TP's wall-clock speedup ($1.2\times$) is modest compared to its convolution call reduction ($2\times$) because LIF dynamics still run per-timestep. DVS-Gesture TAC-TP $K \geq 4$ results are from a single seed. Validation on larger datasets (CIFAR-10, ImageNet) and deeper architectures is needed.

\paragraph{Future directions.} Learnable aggregation weights (replacing fixed exponential decay) could adapt to intermediate temporal structure. Combining TAC-TP with temporal attention~\citep{yao2021temporal} may allow the network to learn which timesteps to share convolution outputs for. Hardware-specific optimization---fusing aggregation and convolution into a single GPU kernel---could further improve TAC-TP's wall-clock speedup.

\paragraph{Conclusion.} We presented a two-part investigation into convolutional SNN acceleration. Five methods for exploiting spike sparsity on GPU all fail due to the fundamental SIMD-sparsity mismatch. Instead, Temporal Aggregated Convolution reduces computation by exploiting convolution linearity across time. TAC (temporal collapse) achieves $13.8\times$ speedup on Apple M3 Max and $11.0\times$ on NVIDIA V100 with $+1.6\%$ accuracy on MNIST. TAC-TP (temporal preservation) achieves $95.1\%$ on DVS-Gesture with $50\%$ fewer convolution calls. The key insight is that the optimal aggregation strategy is data-dependent: collapse for rate-coded data, preserve for event data. All operators are open-source in \texttt{mlx-snn}.

\bibliographystyle{plainnat}
\bibliography{references}

\newpage
\appendix

\section{Proof of Theorem~\ref{thm:tac} (TAC Approximation Error Bound)}
\label{app:proofs}

We prove the error bound by analyzing the membrane potential discrepancy between exact per-step computation and TAC's grouped computation.

\textbf{Setup.} Consider a single convolutional SNN layer processing $T$ timesteps with group size $K$ (assume $K | T$). Let $V_t^\text{exact}$ denote the membrane potential under exact computation:
\begin{equation}
V_t^\text{exact} = \beta V_{t-1}^\text{exact} + W * S_t - V_\text{th} \cdot \Theta(V_{t-1}^\text{exact} - V_\text{th}),
\end{equation}
and $V_t^\text{TAC}$ under TAC:
\begin{equation}
V_{kK}^\text{TAC} = \beta^K V_{(k-1)K}^\text{TAC} + W * A_k - V_\text{th} \cdot \Theta(V_{(k-1)K}^\text{TAC} - V_\text{th}),
\end{equation}
where $A_k = \sum_{j=0}^{K-1} \beta^{K-1-j} S_{kK+j}$.

\textbf{Error decomposition.} The error at group boundary $kK$:
\begin{align}
\epsilon_k &= V_{kK}^\text{exact} - V_{kK}^\text{TAC} \\
&= \beta^K \epsilon_{k-1} + \underbrace{\left(\sum_{j=0}^{K-1} \beta^{K-1-j} W * S_{kK+j} - W * A_k\right)}_{= 0 \text{ by linearity}} + \underbrace{\Delta_\text{spike}(k)}_{\text{spike mismatch}}.
\end{align}

The convolution linearity term vanishes exactly. The spike mismatch $\Delta_\text{spike}(k)$ arises from intermediate spikes within the group:
$|\Delta_\text{spike}(k)| \leq V_\text{th}(K-1)$.

\textbf{Expected error bound.} Since spikes are Bernoulli($\rho$):
\begin{equation}
\E[|\Delta_\text{spike}(k)|^2] \leq V_\text{th}^2 (K-1) \cdot \rho(1-\rho) \cdot N_\text{spatial},
\end{equation}
where $N_\text{spatial} = C \cdot H' \cdot W'$. Unrolling the recurrence yields the stated bound with $C = V_\text{th}^2 N_\text{spatial} / (1 - \beta^{2K})$.

\section{TAC-TP Temporal Resolution Analysis}
\label{app:tactp_analysis}

\begin{proposition}[TAC-TP Temporal Resolution Guarantee]
Let a network have $L$ layers, each applying TAC-TP with group size $K_\ell$. The output temporal dimension at every layer $\ell$ is $T$---the same as the input temporal dimension---regardless of $K_\ell$ values. In contrast, TAC with group sizes $K_\ell$ produces output temporal dimension $T / \prod_{\ell=1}^{L} K_\ell$.
\end{proposition}

\begin{proof}
TAC-TP's inner loop (Algorithm~\ref{alg:tactp}, lines 5--8) generates $K$ independent spike outputs per group, producing exactly $T/K \cdot K = T$ output timesteps per layer. This is independent of $K$ and composes across layers: $L$ layers each preserving $T$ gives output dimension $T$.

TAC (Algorithm~\ref{alg:tac}) generates one spike output per group, giving $T/K$ outputs per layer. Cascading $L$ layers: $T \to T/K_1 \to T/(K_1 K_2) \to \cdots \to T / \prod K_\ell$.
\end{proof}

For DVS-Gesture with $T\!=\!16$ and $L\!=\!5$ layers at $K\!=\!2$: TAC gives $16/2^5 = 0.5$ (i.e., $< 1$ output timestep for the VotingLayer), while TAC-TP gives 16 timesteps. This explains TAC-TP's $+3.8\%$ advantage.

\section{Other Operators: FTC, IMC, TCC Details}
\label{app:other_ops}

\paragraph{Fourier Temporal Convolution (FTC).}
FTC replaces LIF's fixed first-order IIR filter $H(z) = 1/(1 - \beta z^{-1})$ with a learnable biquad:
$H(z) = (b_0 + b_1 z^{-1} + b_2 z^{-2}) / (1 + a_1 z^{-1} + a_2 z^{-2})$,
with BIBO stability guaranteed by $r = \sigma(r_\text{raw}) < 1$. FTC does not reduce convolution calls and introduces IIR filtering overhead.

\paragraph{Information-Theoretic Channel Gating (IMC).}
IMC models each channel as a Binary Symmetric Channel with capacity $C_\text{BSC} = 1 - H_b(\rho)$. Channels with firing rates near $0.5$ are gated out. The minimum channel width for $I$ bits of information is $C_\text{out} \geq I/(1 - H_b(\rho))$. Gating overhead exceeds savings.

\paragraph{Temporal Collapse Convolution (TCC).}
TCC skips convolution during spike-free windows. The probability of a silent frame is $(1-\rho)^{C \cdot H \cdot W}$. For $C\!=\!128, H\!=\!W\!=\!32$: $(0.9)^{131072} \approx 0$. TCC is theoretically exact but requires firing rates so low that entire frames are silent.

\section{Full Per-Seed Results}
\label{app:seeds}

\begin{table}[h]
\caption{Per-seed best test accuracy on MNIST (T=25, 15 epochs).}
\centering
\small
\begin{tabular}{@{}lcccc@{}}
\toprule
Operator & Seed 0 & Seed 1 & Seed 2 & Mean $\pm$ Std \\
\midrule
Baseline & 97.09 & 97.43 & 97.13 & $97.22 \pm 0.15$ \\
TAC $K$=4 & 98.69 & 98.89 & 98.75 & $98.78 \pm 0.08$ \\
TAC $K$=8 & 98.81 & 98.90 & 98.83 & $98.85 \pm 0.04$ \\
TAC $K$=16 & 98.74 & 98.81 & 98.91 & $98.82 \pm 0.07$ \\
TAC-TP $K$=4 & 90.66 & 94.66 & 94.27 & $93.20 \pm 1.78$ \\
TAC-TP $K$=8 & 89.91 & 91.00 & 88.31 & $89.74 \pm 1.10$ \\
TAC-TP $K$=16 & 87.73 & 91.00 & 92.11 & $90.28 \pm 1.82$ \\
\midrule
TCC & 96.95 & 98.00 & 97.64 & $97.53 \pm 0.43$ \\
FTC (order 2) & 96.42 & 96.93 & 96.71 & $96.69 \pm 0.21$ \\
IMC & 97.40 & 97.41 & 97.57 & $97.46 \pm 0.08$ \\
\bottomrule
\end{tabular}
\end{table}

\begin{table}[h]
\caption{Per-seed best test accuracy on Fashion-MNIST (T=25, 15 epochs).}
\centering
\small
\begin{tabular}{@{}lcccc@{}}
\toprule
Operator & Seed 0 & Seed 1 & Seed 2 & Mean $\pm$ Std \\
\midrule
Baseline & 82.07 & 83.91 & 83.45 & $83.14 \pm 0.77$ \\
TAC $K$=4 & 87.47 & 88.38 & 88.08 & $87.98 \pm 0.37$ \\
TAC $K$=8 & 88.48 & 88.21 & 88.81 & $88.50 \pm 0.24$ \\
TAC $K$=16 & 88.14 & 88.54 & 88.35 & $88.34 \pm 0.16$ \\
TAC-TP $K$=4 & 76.82 & 75.95 & 76.38 & $76.38 \pm 0.36$ \\
TAC-TP $K$=8 & 71.70 & 77.85 & 77.09 & $75.55 \pm 2.70$ \\
TAC-TP $K$=16 & 72.37 & 72.74 & 76.38 & $73.83 \pm 1.82$ \\
\midrule
TCC & 84.48 & 85.96 & 85.70 & $85.38 \pm 0.65$ \\
FTC (order 2) & 79.64 & 77.81 & 81.17 & $79.54 \pm 1.37$ \\
IMC & 83.83 & 84.36 & 83.36 & $83.85 \pm 0.41$ \\
\bottomrule
\end{tabular}
\end{table}

\begin{table}[h]
\caption{Per-seed best test accuracy on DVS128-Gesture ($64 \times 64$, 256 epochs with warm restarts).}
\centering
\small
\begin{tabular}{@{}lccccc@{}}
\toprule
Operator & Seed 0 & Seed 42 & Seed 123 & Mean $\pm$ Std & Conv Calls \\
\midrule
Baseline & 95.8 & 96.9 & 96.2 & $96.30 \pm 0.45$ & 80 \\
TAC-TP $K$=2 & 94.1 & 94.8 & 96.5 & $95.13 \pm 1.01$ & 40 \\
TAC-TP $K$=4 & --- & 93.8 & --- & $93.8$ & 20 \\
TAC-TP $K$=8 & --- & 92.0 & --- & $92.0$ & 10 \\
TAC $K$=2 & --- & 91.3 & --- & $91.3$ & 40 \\
\bottomrule
\end{tabular}
\end{table}

\section{Detailed Experimental Configuration}
\label{app:config}

\begin{table}[h]
\caption{Full hyperparameter configuration for all experiments.}
\centering
\small
\begin{tabular}{@{}lcc@{}}
\toprule
Parameter & MNIST / FMNIST & DVS-Gesture \\
\midrule
Timesteps ($T$) & 25 & 16 \\
Batch size & 128 & 16 \\
Learning rate & $10^{-3}$ & $10^{-3}$ \\
LR schedule & Cosine annealing & Cosine w/ warm restarts ($T_\text{max}$=64) \\
Optimizer & Adam & Adam \\
Epochs & 15 & 256 \\
$\beta$ (LIF decay) & 0.9 & 0.5 \\
$V_\text{th}$ (threshold) & 1.0 & 1.0 \\
Surrogate gradient & Fast sigmoid ($\alpha$=25) & Arctan ($\alpha$=2.0) \\
Reset mechanism & Subtract & Subtract (detach reset) \\
Loss function & Cross-entropy on spike count & MSE on firing rate vs.\ one-hot \\
Seeds & 0, 1, 2 & 0, 42, 123 \\
\midrule
\multicolumn{3}{l}{\textit{Dataset-specific}} \\
Input size & $28 \times 28 \times 1$ & $64 \times 64 \times 2$ \\
Spike encoding & Rate (Poisson) & Event binning (by time) \\
Output layer & LIF + spike count & LIF + VotingLayer (10 voters) \\
\bottomrule
\end{tabular}
\end{table}

\section{DVS-Gesture Data Preprocessing}
\label{app:dvsg}

DVS128-Gesture recordings are stored in AEDAT 3.1 format. We parse events $(x, y, t, p)$ with spatial coordinates $(x, y) \in [0, 127]^2$, timestamps $t$ in microseconds, and polarity $p \in \{0, 1\}$. Events are downsampled to $64 \times 64$ by integer division of coordinates by 2. We bin events into $T\!=\!16$ temporal frames by dividing the recording duration equally. Each frame accumulates event counts per spatial position and polarity channel, followed by log-normalization: $f_\text{norm} = \log(1 + f) / \log(1 + f_\text{max})$.

\end{document}